\documentclass[11pt]{article}

\PassOptionsToPackage{colorlinks=true,linkcolor=blue,urlcolor=blue,citecolor=blue}{hyperref}
\usepackage{acl}

\usepackage{times}
\usepackage{latexsym}
\usepackage{makecell}
\usepackage[T1]{fontenc}
\usepackage[utf8]{inputenc}
\usepackage{microtype}
\usepackage{inconsolata}
\usepackage{graphicx}
\usepackage{amsmath}
\usepackage{amssymb} 
\usepackage{amsfonts}
\usepackage{mathtools}
\usepackage{xcolor}
\usepackage{booktabs}
\usepackage{multirow}
\usepackage[table]{xcolor}
\usepackage{colortbl}
\usepackage{listings}
\usepackage{tcolorbox}
\usepackage{tabularx}

\definecolor{lightblue}{HTML}{4A9CC7} 
\definecolor{lightred}{HTML}{E8978A}   

\definecolor{headerblue}{RGB}{200, 220, 240}     
\definecolor{highlightblue}{RGB}{220, 235, 250}    

\definecolor{placeholder}{RGB}{180,50,50}
\definecolor{codebg}{RGB}{250,250,250}
\definecolor{codeframe}{RGB}{120,120,120}

\newcommand\red[1]{\textcolor{red}{#1}}

\newcommand{\lightblue}[1]{\textcolor{lightblue}{#1}}
\newcommand{\lightred}[1]{\textcolor{lightred}{#1}}

\lstdefinestyle{promptstyle}{
    basicstyle=\ttfamily\small,
    breaklines=true,
    breakatwhitespace=false,
    columns=flexible,
    keepspaces=true,
    showstringspaces=false,
    tabsize=2,
    frame=none,
    backgroundcolor=\color{codebg},
    escapeinside={(*@}{@*)},  
}

\usepackage{amsthm}
\newtheorem{theorem}{Theorem}[section]
\newtheorem{lemma}[theorem]{Lemma}
\newtheorem{assumption}[theorem]{Assumption}

\title{Good Reasoning Makes Good Demonstrations: Implicit Reasoning \\ Quality Supervision via In-Context Reinforcement Learning}

\author{
  \textbf{Tiehua Mei}$^{1}$\thanks{Equal contribution.},\ \ \textbf{Minxuan Lv}$^{2}$\footnotemark[1],\ \ \textbf{Leiyu Pan}$^{3}$,\ \ \textbf{Zhenpeng Su}$^{2}$\thanks{Corresponding authors.},\\ \ \ \textbf{Hongru Hou}$^{1}$, 
  \textbf{Hengrui Chen}$^{1}$,\ \ \textbf{Ao Xu}$^{1}$,\ \ \textbf{Deqing Yang}$^{1}$\footnotemark[2] \\[6pt]
  $^{1}$School of Data Science, Fudan University \\
  $^{2}$University of Chinese Academy of Sciences \\
  $^{3}$College of Intelligence and Computing, Tianjin University \\[4pt]
  $\boxtimes$\ \href{mailto:thmei24@m.fudan.edu.cn}{\texttt{thmei24@m.fudan.edu.cn}},\ \href{mailto:suzhenpeng@ucas.ac.cn}{\texttt{suzhenpeng13@163.com}}
}

\begin{document}
\maketitle

\begin{abstract}
Reinforcement Learning with Verifiable Rewards (RLVR) improves reasoning in large language models but treats all correct solutions equally, potentially reinforcing flawed traces that arrive at correct answers by chance. We observe that \emph{better reasoning makes better demonstrations}: high-quality solutions serve as more effective in-context examples than low-quality ones. We term this teaching ability \textbf{Demonstration Utility}, and show that the policy model's own in-context learning ability provides an efficient way to measure it, yielding a quality signal termed \textbf{Evidence Gain}. To leverage this signal during training, we introduce \textbf{In-Context RLVR}, which prepends demonstrations before each rollout. Theoretically, we prove that this simple input modification implicitly reweights rewards by a factor approximately proportional to Evidence Gain, assigning higher weights to high-quality traces without requiring costly computation. Experiments on mathematical reasoning benchmarks demonstrate consistent improvements in both accuracy and reasoning quality over standard RLVR baselines. Our codes and datasets are available at \url{https://github.com/Mithas-114/IC-DAPO}.
\end{abstract}

\section{Introduction}
\label{sec:intro}

Reinforcement Learning with Verifiable Rewards (RLVR) has emerged as a powerful paradigm for improving LLM reasoning~\citep{shao2024deepseekmathpushinglimitsmathematical,guo2025deepseekr1}, especially in domains such as mathematics where correctness can be checked by rules~\citep{su2025entropyratioclippingsoft}. By using outcome-level supervision, RLVR avoids costly process annotations and scales well~\citep{mroueh2025reinforcementlearningverifiablerewards}. However, this simplicity comes with a limitation: all correct solutions receive equal reward, regardless of the reasoning used to obtain them~\citep{do2025definesgoodreasoningllms}. This is problematic since models can produce flawed reasoning traces that coincidentally get correct answers, particularly when final answers are simple values that can be guessed~\citep{guo2025rightenoughpitfallsoutcome}. Consequently, reinforcing such traces may corrupt internal reasoning strategies, degrading performance on other problems~\citep{macdiarmid2025naturalemergentmisalignmentreward}.

A natural solution is to use process reward models (PRMs)~\citep{zhang2025lessonsdevelopingprocessreward,ye2025correctnessharmonizingprocessoutcome} that score intermediate steps. However, PRMs typically require extensive human annotation or auxiliary trained evaluators~\citep{lightman2024let}. This raises a key question: \emph{Can we encourage high-quality reasoning within RLVR without requiring step-level supervision or trained reward models?}

\paragraph{Demonstration Utility as Global Quality Signal.} Our key insight is that \emph{high-quality reasoning traces are better teachers than low-quality ones}~\citep{min2022rethinking}. Consider two solutions that both arrive at the correct answer: one is coherent and complete; the other contains redundant or unclear steps. When used as in-context demonstrations, the former provides transferable problem-solving patterns that help the model generate better solutions, while the latter provides less reference value~\citep{li2025llmseasilylearnreason}. We term this teaching ability \textbf{Demonstration Utility}. Crucially, the policy model's own in-context learning (ICL) ability provides a natural way to measure Demonstration Utility. Specifically, we construct a held-out validation set composed of questions and high-quality reference reasoning traces. We propose computing the average increase in the model's log-likelihood of generating these references after a candidate reasoning trace is prepended as a demonstration. We call this measure \textbf{Evidence Gain} (\S\ref{sec:evidence_gain}). Unlike PRMs that require external evaluators, Evidence Gain leverages the intrinsic ICL capability of the policy model itself. Experiments in Section~\ref{sec:evidence_gain} confirm that this intrinsic signal effectively distinguishes good reasoning from bad.

\paragraph{Implicit Reward Reweighting via In-Context RLVR.}
While Evidence Gain provides a reasoning quality signal, computing it as rewards would introduce substantial overhead. Fortunately, we show that this explicit computation is unnecessary. Our key idea is to reverse the process: instead of computing Evidence Gain \emph{after} generation as rewards, we use the same validation set to guide training \emph{before} generation. Specifically, before each rollout, we sample a demonstration from the validation set and prepend it to the current question, then perform standard RL updates in this demonstration-conditioned setting, a procedure we term \textbf{In-Context RLVR}.
Theoretically, we prove that this training objective is equivalent to standard zero-shot RLVR but with \emph{rewards implicitly reweighted by a factor approximately proportional to $\exp(\Delta)$}, where $\Delta$ denotes Evidence Gain (Theorem~\ref{thm:implicit_reweight}, Theorem~\ref{thm:log_linear}). Consequently, high-quality traces with greater teaching utility receive amplified gradient signals, while low-quality traces receive relatively lower weights through this implicit reweighting mechanism.

\paragraph{Contributions.}
\textbf{(1)} We introduce \textbf{Evidence Gain}, a quality signal that measures reasoning quality by leveraging the policy model's intrinsic ICL ability, requiring no external evaluators or step-level supervision. 
\textbf{(2)} We show that this signal can be seamlessly integrated into training via \textbf{In-Context RLVR}, which prepends demonstrations during training. We prove (Theorem~\ref{thm:implicit_reweight}) that this simple input modification implicitly reweights rewards, and further show (Theorem~\ref{thm:log_linear}) that the weight factor is approximately proportional to $\exp(\Delta)$.
\textbf{(3)} Experiments across mathematical benchmarks validate that our method improves both accuracy and reasoning quality over competitive baselines, while introducing less than 5\% training overhead.
\section{Evidence Gain as Quality Measure}
\label{sec:evidence_gain}

This section formally defines \textbf{Evidence Gain}, a quality signal that measures reasoning quality by leveraging the policy model's intrinsic in-context learning ability, and validates it empirically.

Our basic idea is that, when used as demonstrations, high-quality reasoning traces provide more valuable problem-solving patterns, while low-quality reasoning (even with correct answers) provides less reference value due to flaws such as inconsistent logic~\citep{li2025llmseasilylearnreason}. This motivates us to quantify the quality of a solution by its teaching ability as a demonstration, formalized as follows.

\subsection{Definition}
Let $\pi_\theta$ denote policy model. Given a question $q$, a model-generated reasoning trace $r$, and a held-out validation set $\mathcal{E} = \{(e_q, e_r)\}$ composed of questions $e_q$ and high-quality reference reasoning traces $e_r$, we define \textbf{Evidence Gain} as:
\begin{equation}
\small
\Delta(q, r) = \mathbb{E}_{e \sim \mathcal{E}}\left[\log \pi_\theta(e_r | q, r, e_q) - \log \pi_\theta(e_r | e_q)\right]
\label{eq:evidence_gain}
\end{equation}
Intuitively, $\Delta$ measures how much prepending the pair $(q, r)$ improves the model's ability to generate reference solutions. Averaging over $\mathcal{E}$ ensures that high $\Delta$ reflects transferable reasoning patterns rather than spurious matches to any single sample.

\subsection{Correlation with Reasoning Quality}
We validate Evidence Gain on DeepSeek-R1-Distill-Qwen at 1.5B and 7B scales using KlearReasoner-MathSub-30K dataset~\citep{su2025klearreasoneradvancingreasoningcapability}. First, we sample 3,000 questions, generate 8 responses per question, and retain traces with correct final answers (12,251 for 1.5B and 16,910 for 7B). We then employ DeepSeek-V3.2~\citep{deepseekai2025deepseekv32pushingfrontieropen}, a strong LLM-based evaluator, to assess the reasoning quality. The evaluator scores each solution across multiple dimensions including logical coherence and redundancy, and assigns an overall quality score on a 1--5 scale. Next, we sample 100 new questions to construct a held-out validation set $\mathcal{E}$. For each $e_q\in \mathcal{E}$, we generate a correct solution using DeepSeek-R1-0528~\citep{guo2025deepseekr1}. We leverage a feature of DeepSeek-R1: it first produces a draft chain-of-thought inside \texttt{<think>} \texttt{</think>}, and then outputs a more polished reasoning solution afterward. We treat the content following \texttt{</think>} as the high-quality reference trace $e_r$.

\begin{figure}[t]
\centering
\includegraphics[width=\columnwidth]{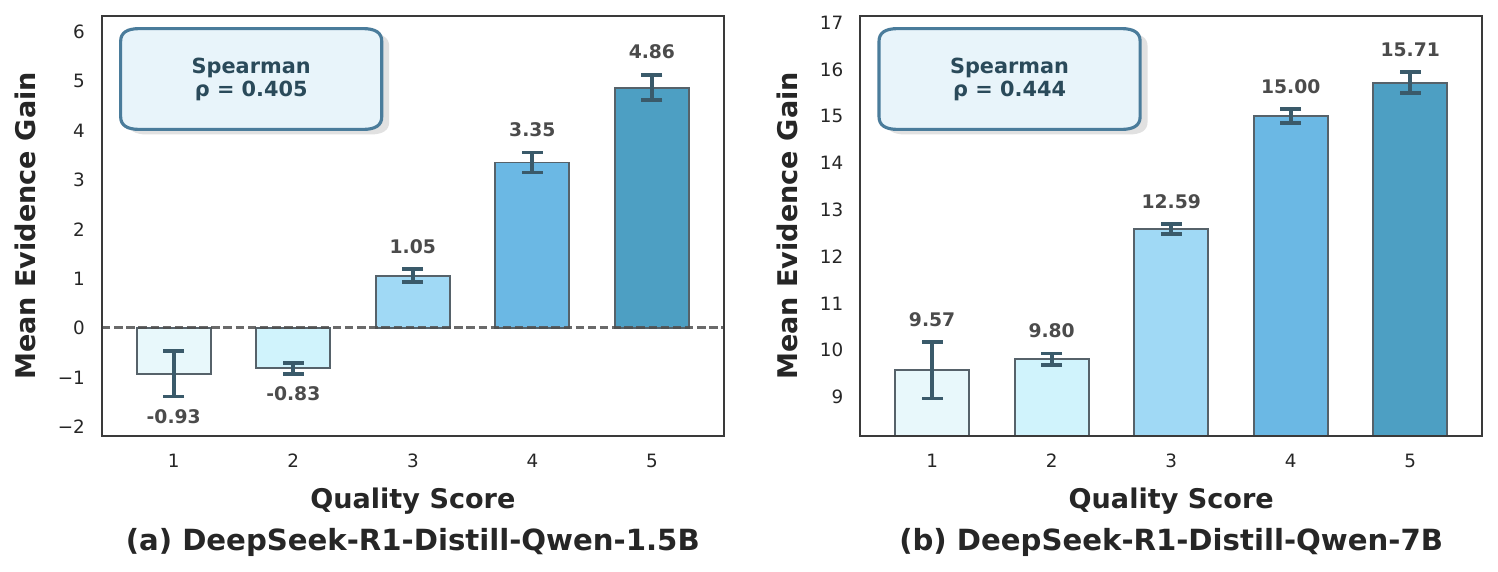}
\caption{Mean Evidence Gain by quality score with 95\% confidence intervals on two models.}
\label{fig:evidence_gain}
\end{figure}

Figure~\ref{fig:evidence_gain} shows an interesting pattern: model ability determines the absolute baseline of $\Delta$, while reasoning quality differentiates its relative magnitude. The 7B model, with stronger ICL ability, extracts useful information from any reasoning trace, resulting in uniformly positive $\Delta$ values greatly higher than those of 1.5B. However, the relative ordering remains consistent within each scale: high-quality traces yield higher $\Delta$ than low-quality ones. This relative difference is what matters for RL training. These results confirm that Evidence Gain effectively distinguishes reasoning quality. Human evaluation in Appendix~\ref{app:human_eval} supports these findings.

\subsection{Comparison with Other Proxy Signals}
To contextualize Evidence Gain, we compare it against three representative proxy signals of reasoning quality from prior work: response length~\cite{xin2026surrogatesignalsformatlength}, log-probability~\cite{kadavath2022languagemodelsmostlyknow}, and majority voting~\citep{wang2023selfconsistencyimproveschainthought}. Detailed descriptions of these signals are provided in Appendix~\ref{app:proxy_signals}. Table~\ref{tab:proxy_comparison} reports Spearman correlations with reasoning quality scores.

\begin{table}[t]
\centering
\small
\begin{tabular}{lcccc}
\toprule
 & Length & LogProb & MajorVote & $\Delta$ (Ours) \\
\midrule
1.5B & $-$0.147 & 0.129 & 0.079 & \textbf{0.405} \\
7B & $-$0.161 & 0.178 & 0.109 & \textbf{0.444} \\
\bottomrule
\end{tabular}
\caption{Spearman correlation ($\rho$) between proxy signals and reasoning quality. Evidence Gain achieves stronger correlation than alternatives.}
\label{tab:proxy_comparison}
\end{table}

Evidence Gain achieves stronger correlation with reasoning quality than all three proxies. This is because Evidence Gain captures transferable problem-solving patterns through the model's ICL mechanism, encoding richer quality information than any single surface-level feature (e.g., reasoning length). This is further validated by our per-dimension analysis in Appendix~\ref{app:evidence_gain_validation}, which shows that Evidence Gain reflects multiple aspects of reasoning quality rather than a single dimension.
\section{Implicit Reward Reweighting}
\label{sec:method}

The correlation between Evidence Gain ($\Delta$) and reasoning quality (\S\ref{sec:evidence_gain}) suggests that upweighting high-$\Delta$ traces during training could improve reasoning. However, explicitly computing $\Delta$ for each rollout is prohibitively expensive.\footnote{Computing $\Delta$ for ${\sim}$12K samples over 100 demonstrations requires approximately 80 hours on H800.} In this section, we show that explicit computation is unnecessary. Evidence Gain can be seamlessly integrated into training through \textbf{In-Context RLVR}.

In-Context RLVR modifies standard RLVR by a simple input-side change. The core idea is to reverse the process: instead of computing Evidence Gain \emph{after} generation to reweight rewards, we utilize the demonstration set to guide training \emph{before} generation. Specifically, before each rollout, we sample a demonstration $e = (e_q, e_r)$ uniformly at random from the demonstration set $\mathcal{E}$, prepend it to the current question $q$, and generate response $r$ from the demonstration-conditioned policy $\pi_\theta(r|e,q)$. Standard RL updates are then performed using correctness reward $R(q,r) \in \{0,1\}$.

\subsection{Theoretical Foundation}
We establish that In-Context RLVR implicitly performs reward reweighting, where the reweighting factor is approximately proportional to $\exp(\Delta)$. We first define the In-Context RLVR objective:
\begin{equation}
J_{\text{IC}}(\theta) = \mathbb{E}_{q} \mathbb{E}_{e \sim \mathcal{E}, r \sim \pi_\theta(\cdot | e, q)}[R(q, r)].
\label{eq:icrlvr_obj}
\end{equation}
Our theoretical result proceeds in two steps (full derivation in Appendix~\ref{app:proof}).

\paragraph{Step 1: Bayesian Identity.} Since demonstrations are sampled independently from training questions, the demonstration question $e_q$ alone carries no information relevant to solving $q$.\footnote{This assumption is empirically verified in Appendix~\ref{app:bayesian_identity}.} Under this assumption, the following Bayesian identity holds:
\begin{equation*}
\pi_\theta(r|e,q) = \pi_\theta(r|q) \cdot \frac{\pi_\theta(e_r|q, r, e_q)}{\pi_\theta(e_r|e_q)}.
\end{equation*}
This identity shows that the conditioned policy equals the base policy multiplied by a likelihood ratio. Using this, we prove (Theorem~\ref{thm:implicit_reweight}) that the In-Context RLVR objective can be rewritten as:
\begin{equation}
J(\theta) = \mathbb{E}_{q} \mathbb{E}_{r \sim \pi_\theta(\cdot | q)}\left[R(q, r) \cdot w(q, r)\right],
\label{eq:reweighted_obj}
\end{equation}
where $w(q, r) = \mathbb{E}_{e}[\exp(\Delta_e)]$ with $\Delta_e = \log \pi_\theta(e_r|q,r,e_q) - \log \pi_\theta(e_r|e_q)$.

\paragraph{Step 2: Log-Linear Relationship.} The weight $w = \mathbb{E}_e[\exp(\Delta_e)]$ differs from $\exp(\Delta)$ where $\Delta = \mathbb{E}_e[\Delta_e]$ is the Evidence Gain. However, we further show (Theorem~\ref{thm:log_linear}) that $$\log w(q,r) \approx \Delta(q,r) + c,$$ where $c$ is a model-specific constant. This log-linear approximation is empirically tight: Pearson correlations between $\log w$ and $\Delta$ exceed 0.8 at both 1.5B and 7B scales (Figure~\ref{fig:log_linear}), confirming that $w(q,r) \propto \exp(\Delta(q,r))$ in practice.

Combining these two steps, In-Context RLVR implicitly reweights rewards by a factor $w$ approximately proportional to $\exp(\Delta)$, assigning higher weights to traces with higher Evidence Gain.

\begin{table*}[t]
\centering
\setlength{\tabcolsep}{14pt}
\resizebox{1.0\linewidth}{!}{
\begin{tabular}{lccccccc|c}
\toprule
\textbf{Method} & \textbf{AIME24} & \textbf{AIME25} & \textbf{HMMT25} & \textbf{MATH500} & \textbf{AMC23} & \textbf{Olympiad} & \textbf{Average} &\textbf{Time/Step (s)} \\
\midrule
\textbf{DS-R1-Distill-Qwen-1.5B} & 29.2 & 24.1 & 13.1 & 86.0 & 73.7 & 51.8 & 46.3 & -- \\
\quad + GRPO~\citep{shao2024deepseekmathpushinglimitsmathematical}  & 33.4 & 28.1 & 16.6 & 88.3 & 79.3 & 56.2 & 50.3 & 457.4 \\
\rowcolor{highlightblue}
\quad + IC-GRPO (Ours) & 38.3 \rlap{\red{$\boldsymbol\uparrow$}} & 30.6 \rlap{\red{$\boldsymbol\uparrow$}} & 17.7 \rlap{\red{$\boldsymbol\uparrow$}} & 89.5 \rlap{\red{$\boldsymbol\uparrow$}} & 82.5 \rlap{\red{$\boldsymbol\uparrow$}} & 56.9 \rlap{\red{$\boldsymbol\uparrow$}} & 52.6 \rlap{\red{$\boldsymbol\uparrow$}} & 461.8 \\
\quad + DAPO~\citep{yu2025dapoopensourcellmreinforcement}  & 40.0 & 28.4 & 19.2 & 90.0 & 84.4 & 61.6 & 53.9& 459.6 \\
\quad + CISPO~\citep{minimax2025minimaxm1scalingtesttimecompute} & 32.9 & 25.1 & 13.2 & 85.8 & 80.9 & 54.9 & 48.8& 466.3\\
\quad + GSPO~\citep{zheng2025groupsequencepolicyoptimization} & 42.5 & \underline{33.6} & 19.0 & 90.3 & \underline{85.9} & \underline{62.6} & 55.7& 437.3\\
\quad + CE-GPPO~\citep{su2025cegppocoordinatingentropygradientpreserving} & 42.8 & 32.5 & \underline{20.5} & \underline{91.0} & 85.8 & 61.8 &55.7& 464.0\\
\rowcolor{highlightblue}
\quad + IC-DAPO (Ours) & \textbf{45.6} \rlap{\red{$\boldsymbol\uparrow$}}& \textbf{34.2} \rlap{\red{$\boldsymbol\uparrow$}}& \textbf{19.7} \rlap{\red{$\boldsymbol\uparrow$}}& \textbf{90.6} \rlap{\red{$\boldsymbol\uparrow$}}& \textbf{86.2} \rlap{\red{$\boldsymbol\uparrow$}}& \textbf{62.1} \rlap{\red{$\boldsymbol\uparrow$}}& \textbf{\underline{56.4}} \rlap{\red{$\boldsymbol\uparrow$}} & 477.2\\
\midrule
\textbf{DS-R1-Distill-Qwen-7B} & 54.5 & 39.1 & 26.2 & 93.6 & 90.6 & 67.0 & 61.8 & -- \\
\quad + GRPO~\citep{shao2024deepseekmathpushinglimitsmathematical}  & 55.3 & 40.3 & 24.5 & 93.7 & 88.8 & 65.6 & 61.4&305.6 \\
\quad + DAPO~\citep{yu2025dapoopensourcellmreinforcement}  & 62.0 & 45.9 & 27.4 & 94.1 & 92.3 & 69.9 & 65.3 & 303.1 \\
\quad + CE-GPPO~\cite{su2025cegppocoordinatingentropygradientpreserving} & \underline{64.2} & \textbf{50.3} & \underline{28.9} & \underline{95.3} & \underline{93.3} & \underline{71.6} & \underline{67.3} & 292.5\\
\rowcolor{highlightblue}
\quad + IC-DAPO (Ours)  & \textbf{66.5} \rlap{\red{$\boldsymbol\uparrow$}} & \underline{49.8} \rlap{\red{$\boldsymbol\uparrow$}} & \textbf{29.4} \rlap{\red{$\boldsymbol\uparrow$}}&\textbf{95.6} \rlap{\red{$\boldsymbol\uparrow$}} &  \textbf{93.7} \rlap{\red{$\boldsymbol\uparrow$}} & \textbf{71.7} \rlap{\red{$\boldsymbol\uparrow$}} & \textbf{67.8} \rlap{\red{$\boldsymbol\uparrow$}} & 315.6\\
\bottomrule
\end{tabular}
}
\caption{Performance comparison across mathematical reasoning benchmarks. \textbf{Bold} and \underline{underline} indicate the best and second-best results respectively. \red{$\boldsymbol\uparrow$} denotes improvement over the corresponding baseline (GRPO or DAPO). Notably, training times are incomparable across scales due to different GPU configurations (32 vs. 128 GPUs). Given this high GPU requirements at 7B scales, GSPO and CISPO are evaluated only at 1.5B.}
\label{tab:table1}
\end{table*}

\subsection{Interpretation of Implicit Reweighting}
The reweighted reward $R \cdot w$ in Eq.~\ref{eq:reweighted_obj} implies a two-stage selection mechanism. First, the binary reward $R$ filters out traces with incorrect answers. Second, among correct traces, the weight $w \propto \exp(\Delta)$ differentiates reasoning quality, assigning higher weights to high-quality traces and lower weights to low-quality ones. While training explicitly samples from $\pi_\theta(r|e, q)$, In-Context RLVR implicitly optimizes the base policy $\pi_\theta(\cdot | q)$ with rewards reweighted by Evidence Gain.

Eq.~\ref{eq:reweighted_obj} shows that our method employs the model's own ICL ability to guide optimization, with the policy serving as both the learner and the implicit quality evaluator. A natural concern is whether Evidence Gain remains a valid quality signal as the policy evolves, since our validation in Section~\ref{sec:evidence_gain} uses a fixed model. We address this in Section~\ref{sec:exp}, showing that the correlation between Evidence Gain and reasoning quality remains stable throughout training.

Notably, while $J_{\text{IC}}$ and $J$ share the same expectation, they differ in variance. The explicit reweighting term $w \propto \exp(\Delta)$ in $J$ would introduce prohibitive reward variance. $J_{\text{IC}}$ avoids this instability by shifting the sampling distribution directly.
\section{Experiments}
\label{sec:exp}

To validate our framework, we combine In-Context RLVR with DAPO~\citep{yu2025dapoopensourcellmreinforcement}, yielding \textbf{IC-DAPO}. We choose DAPO as backbone because it is a widely adopted RLVR method whose key techniques (e.g., clip-higher) have been incorporated into many subsequent methods~\citep{yue2025vapoefficientreliablereinforcement, su2025cegppocoordinatingentropygradientpreserving}, making it a representative baseline for evaluating input-side modifications. All details of this section are provided in Appendix~\ref{app:exp_details}.

\subsection{Setup}

\paragraph{Dataset.}
Our training data is derived from KlearReasoner-MathSub-30K~\citep{su2025klearreasoneradvancingreasoningcapability}, which contains 30K mathematical reasoning problems. We partition training data into three disjoint subsets: (1) a \textbf{training set} for policy optimization, (2) a \textbf{demonstration set} $\mathcal{E}$ containing 1,082 question-reasoning pairs used for demonstration during IC-DAPO training, and (3) a \textbf{held-out set} $\mathcal{E}_0$ of 100 additional examples reserved for the correlation analysis in \S\ref{sec:analysis}. Both $\mathcal{E}$ and $\mathcal{E}_0$ are constructed following the procedure described in \S\ref{sec:evidence_gain}.

\paragraph{Baselines.}
We compare against several popular RLVR methods, including both standard outcome-based algorithms and more advanced objective-modifying variants. This selection directly tests whether our input-side modification can match algorithmic innovations in policy optimization. We exclude PRM-based methods as they require costly overhead that our method aims to avoid.

\paragraph{Training and Evaluation.}
We train DeepSeek-R1-Distill-Qwen at 1.5B and 7B scales. We conduct evaluations across various authoritative mathematical reasoning benchmarks, including AIME24, AIME25, HMMT25, MATH500~\citep{lightman2024let}, AMC23 and OlympiadBench~\citep{he-etal-2024-olympiadbench}. Following ~\citet{su2025cegppocoordinatingentropygradientpreserving}, we report avg@4 scores on MATH500 and OlympiadBench, and avg@32 scores on all other benchmarks. Crucially, \emph{all evaluation is conducted in zero-shot mode}, ensuring fair comparison with baselines.

\subsection{Main Results}
Table~\ref{tab:table1} presents benchmark performance. IC-DAPO outperforms DAPO by +2.5 average points at both scales, with gains particularly pronounced on competition benchmarks: +5.6 on AIME24 and +5.8 on AIME25 for the 1.5B model. This supports our hypothesis that implicit quality reweighting helps more on challenging problems where correct-but-low-quality traces are most harmful, which is further validated in Section~\ref{sec:q4}. To verify that this improvement generalizes beyond DAPO, we also apply In-Context RLVR to GRPO at 1.5B scale. IC-GRPO achieves +2.3 average improvement over GRPO across all benchmarks at this scale, confirming that the implicit reweighting mechanism transfers across different RL optimizers.

Beyond improvements over the corresponding baselines, IC-DAPO also matches or exceeds methods that modify the RL objective (e.g., GSPO, CISPO), achieving the highest average score at both scales while only altering the \emph{input distribution}. This suggests that input-side modification constitutes an improvement axis orthogonal to policy optimization algorithms. We further compare wall-clock training time per training step and find that IC-DAPO incurs slight overhead ($<$5\%), confirming its practicality.

\subsection{Analysis}
\label{sec:analysis}

Our theory (\S\ref{sec:method}) predicts that In-Context RLVR implicitly upweights high-$\Delta$ traces. To verify this, we track training dynamics by computing $\Delta$ on the held-out set $\mathcal{E}_0$ and assessing reasoning quality via Deepseek-V3.2, following procedures in \S\ref{sec:evidence_gain}.

\paragraph{Q1: Does implicit reweighting occur?}
Figure~\ref{fig:dynamics} (left) shows that mean Evidence Gain increases steadily under IC-DAPO throughout training, while DAPO exhibits smaller and slower growth. This confirms that the conditioned objective steers the policy toward traces with higher demonstration utility, exactly as predicted.

\paragraph{Q2: Does this improve reasoning quality?}
Figure~\ref{fig:dynamics} (middle) shows higher $\Delta$ corresponds to improved quality scores. Note that $\mathcal{E}_0$ used for evaluation is disjoint from $\mathcal{E}$ used for training, ensuring unbiased evaluation. This shows that by upweighting traces with high teaching utility, we encourage better reasoning rather than merely correct answers.

\paragraph{Q3: Is Evidence Gain valid throughout training?}
Finally, we address the concern from \S\ref{sec:method}. Figure~\ref{fig:dynamics} (right) shows that the Spearman correlation between $\Delta$ and quality remains stable (around $\rho \approx 0.4$) across training steps, confirming that the policy model's intrinsic ICL signal remains a robust quality indicator as training progresses.

\begin{figure}[t]
\centering
\includegraphics[width=\columnwidth]{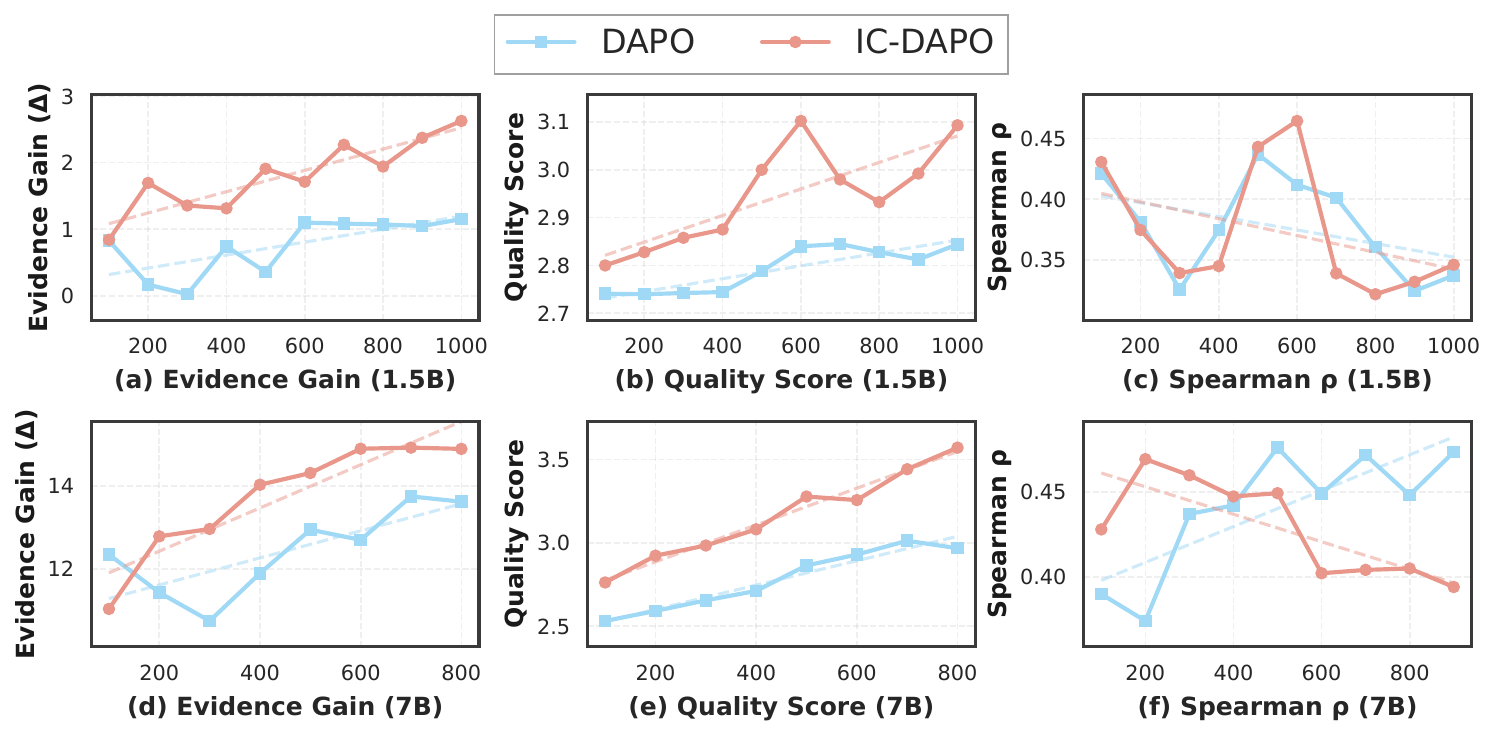}
\caption{Training dynamics of Evidence Gain, quality score, and their correlation ($\rho$) across training steps.}
\label{fig:dynamics}
\end{figure}

\begin{table}[t]
\centering
\small
\setlength{\tabcolsep}{4pt}
\begin{tabular}{llccc}
\toprule
\textbf{Method} & \textbf{Scale} & \textbf{Easy} & \textbf{Medium} & \textbf{Hard} \\
\midrule
DAPO & 1.5B & 98.3 & 90.1 & 23.1 \\
IC-DAPO & 1.5B & 98.8 {\scriptsize(+0.5\%)} & 93.5 {\scriptsize(+3.8\%)} & 26.0 {\scriptsize(\textbf{+12.6\%})} \\
\midrule
DAPO & 7B & 98.6 & 97.8 & 39.2 \\
IC-DAPO & 7B & 99.3 {\scriptsize(+0.7\%)} & 98.2 {\scriptsize(+0.4\%)} & 43.2 {\scriptsize(\textbf{+10.2\%})} \\
\bottomrule
\end{tabular}
\caption{Difficulty-stratified analysis. Parenthetical values denote relative improvement over DAPO.}
\label{tab:difficulty}
\end{table}

\paragraph{Q4: Where do the gains come from?}\label{sec:q4}
To examine whether quality-based reweighting benefits problems that demand higher reasoning quality, we stratify evaluation by difficulty. We rank all problems from 6 benchmarks in Table \ref{tab:table1} by backbone model accuracy and divide them into equal thirds. Table~\ref{tab:difficulty} shows that gains concentrate on hard problems: +12.6\% relative improvement at 1.5B and +10.2\% at 7B. On easier problems where baseline accuracy is near-perfect, there is limited room for quality-based reweighting to contribute. This pattern confirms that implicit quality reweighting provides the greatest advantage precisely where reasoning quality matters most.
\section{Conclusions}

We introduce Evidence Gain, a quality signal that measures reasoning quality based on the policy model's intrinsic ICL ability. To leverage this signal, we propose In-Context RLVR, which prepends demonstrations during training and implicitly reweights rewards by Evidence Gain to encourage high-quality traces. Experiments confirm improvements in both accuracy and reasoning quality over standard RLVR, providing a practical approach to improve reasoning quality in RLVR.

\section*{Limitations}
This work has two main limitations. First, although In-Context RLVR demonstrates consistent improvements across mathematical reasoning benchmarks, its generalization to other reasoning-intensive domains such as STEM problem-solving remains an open question due to computational constraints. Second, constructing the demonstration set requires access to a strong model (e.g., DeepSeek-R1) for generating high-quality reference traces. Alternative construction strategies that reduce this dependency should be further developed.

\bibliography{custom}

\appendix
\section{Related Work}
\label{sec:related}

\paragraph{Reinforcement Learning with Verifiable Rewards.}
RLVR has become a dominant paradigm for improving LLM reasoning~\citep{lambert2025tulu3pushingfrontiers,guo2025deepseekr1}.
By using rule-based correctness signals, RLVR avoids reward hacking and scales well.
GRPO~\citep{shao2024deepseekmathpushinglimitsmathematical} optimizes policies via group-based reward normalization without requiring a separate critic model.
Subsequent work addresses entropy collapse~\citep{yu2025dapoopensourcellmreinforcement}, sequence-level optimization~\citep{zheng2025groupsequencepolicyoptimization}, and training stability~\citep{minimax2025minimaxm1scalingtesttimecompute}.
Despite the advances, a fundamental limitation persists.
Binary correctness rewards assign equal reward to correct solutions regardless of reasoning quality, potentially reinforcing spurious traces that get correct answers through flawed logic~\citep{zhang2025lessonsdevelopingprocessreward,macdiarmid2025naturalemergentmisalignmentreward}.

\paragraph{Process Reward Models.}
PRMs address this limitation by providing step-level feedback.
\citet{lightman2024let} demonstrate substantial gains from process supervision over outcome supervision, though their approach required approximately 800,000 human-annotated step labels.
Automated alternatives such as Math-Shepherd~\citep{wang2024mathshepherdverifyreinforcellms} construct step-level labels through Monte Carlo estimation, trading annotation cost for computational overhead from repeated rollouts.
These methods share a common constraint.
Obtaining reliable process signals demands either substantial human effort or significant compute resources.

\paragraph{Quality-Aware Reward Reweighting.}
\label{app:reweighting_taxonomy}
Standard RLVR assigns equal reward to all correct traces regardless of reasoning quality. Several approaches address this limitation through reward reweighting, which we categorize into three families.

\textit{External reward models} (e.g., PRM~\citep{zhang2025lessonsdevelopingprocessreward}) train a separate model to score reasoning quality at each step. However, this requires enormous annotation effort, adds per-step evaluation cost throughout training, and remains susceptible to reward hacking~\citep{gao2022scalinglawsrewardmodel}.

\textit{Self-evaluation methods} (e.g., Self-Rewarding~\citep{selfrewarding}, RLPR~\citep{yu2025rlprextrapolatingrlvrgeneral}) let the model judge the quality of its own output. This avoids an external model but introduces different costs: each quality assessment requires additional forward passes during training. Moreover, without reliable external guidance, the model may reinforce its own errors through self-consistent illusions~\citep{huang2024largelanguagemodelsselfcorrect}.

\textit{Proxy signals} (e.g., length~\citep{xin2026surrogatesignalsformatlength}, log-probability~\citep{kadavath2022languagemodelsmostlyknow}, majority voting~\citep{wang2023selfconsistencyimproveschainthought}) are lightweight but capture only a single surface dimension of quality. Majority voting, for instance, measures only final-answer consistency and cannot distinguish rigorous reasoning from lucky guessing.

Our method, Evidence Gain, belongs to the proxy signal family. The key difference from existing proxy signals is that Evidence Gain integrates multiple quality dimensions into a single scalar through the model's ICL mechanism, rather than relying on one surface feature. We quantitatively compare Evidence Gain against representative proxy signals in Section~\ref{sec:evidence_gain}.

\paragraph{Self-Derived Reasoning Quality Signals.}
An emerging direction leverages signals derived from the model itself rather than external supervision.
\citet{zhang2025covo} construct intrinsic rewards from trajectory consistency and volatility, which requires computing distances between each intermediate state and all distinct final answers across sampled rollouts, incurring $O(NTK)$ additional forward passes per prompt beyond standard algorithm.
\citet{xiong2025selfrewardingcorrectionmathematicalreasoning} train models to perform iterative self-correction loops, where the model sequentially detects errors, revises outputs, and decides when to terminate.
In contrast, our method measures reasoning quality through demonstration utility, motivated by findings that high-quality reasoning traces can serve as effective in-context demonstrations~\citep{min2022rethinking,li2025llmseasilylearnreason}.
Crucially, this signal can be integrated implicitly into the training objective via In-Context RLVR, adding less than 5\% overhead without any explicit quality computation. Importantly, our approach differs fundamentally from demonstration selection methods~\citep{luo2024incontextlearningretrieveddemonstrations}, which develop retrieval or optimization strategies to identify the best demonstrations for each input query. In contrast, we leverage demonstrations to assess the quality of candidate reasoning traces; since Evidence Gain (Eq.~\ref{eq:evidence_gain}) is defined as an average over the validation set, all demonstrations are equally important for every query in our method.

\section{Experimental Details}
\label{app:exp_details}

\subsection{Datasets and Preprocessing}
\label{app:datasets}

Our training data is derived from KlearReasoner-MathSub-30K~\citep{su2025klearreasoneradvancingreasoningcapability}, which contains approximately 30K high-quality mathematical reasoning problems collected from several curated sources, including Skywork-OR1~\citep{he2025skyworkopenreasoner1}, Acereason~\citep{chen2025acereasonnemotronadvancingmathcode}, NuminaMath~\citep{numina_math_datasets}, and DeepScaleR~\citep{deepscaler2025}. To mitigate potential data contamination, the dataset has been processed with 9-gram deduplication against the evaluation benchmarks.

We first randomly partition 28k samples from the full dataset for policy optimization. For the remaining 2k samples, we generate reasoning traces using DeepSeek-R1-0528~\citep{guo2025deepseekr1} and filter the outputs using rule-based validators to retain only those with correct final answers, yielding approximately 1,200 valid examples. From this filtered set, we randomly select 100 examples to form the held-out set $\mathcal{E}_0$ for the correlation analysis in Section~\ref{sec:analysis}, with the remaining 1,082 examples forming the demonstration set $\mathcal{E}$ used for prepending demonstrations during training. For each example in $\mathcal{E}$ and $\mathcal{E}_0$, we extract the content following \texttt{</think>} as the reference reasoning trace following the procedure described in Section~\ref{sec:evidence_gain}. We manually verify the quality of these solutions.

\paragraph{Demonstration Pipeline Cost.}
The demonstration construction pipeline is lightweight and reproducible. The entire process involves 2,000 API calls to DeepSeek-R1, correctness verification via math-verify, and extraction of the content following \texttt{</think>}. This takes approximately one hour at a total API cost of approximately \$50. This is a one-time, pre-training cost with zero expense during training or inference. For comparison, training a PRM requires many step-level annotations. Scaling to new domains requires only replacing the source problems, with no changes to the pipeline itself.

\subsection{Details for Main Experiments}
\label{app:main_exp_details}

We train DeepSeek-R1-Distill-Qwen-1.5B\footnote{https://huggingface.co/deepseek-ai/DeepSeek-R1-Distill-Qwen-1.5B} and DeepSeek-R1-Distill-Qwen-7B\footnote{https://huggingface.co/deepseek-ai/DeepSeek-R1-Distill-Qwen-7B}. We evaluate on six mathematical reasoning benchmarks: AIME24, AIME25, HMMT25, MATH500, AMC23, and OlympiadBench. For evaluation metrics, we report avg@4 scores on MATH500~\citep{lightman2024let} and OlympiadBench~\citep{he-etal-2024-olympiadbench}, and avg@32 scores on all other benchmarks, following prior work~\citep{su2025cegppocoordinatingentropygradientpreserving}. At inference, we set the maximum generation length to 32k tokens for AIME24 and AIME25, and 16k tokens for the other datasets. For answer extraction, we follow the standard practice adopted in \citet{yang2024qwen25mathtechnicalreportmathematical}: parsing the contents enclosed within the \texttt{\textbackslash boxed\{\}} structure in model outputs to identify the final answer. Answer correctness is judged by math-verify\footnote{\url{https://github.com/huggingface/Math-Verify}}, which performs symbolic comparison to handle equivalent mathematical expressions.

All evaluation is conducted in zero-shot mode without any demonstrations, ensuring fair comparison with baseline methods and validating that our approach requires no modification to deployment. This zero-shot evaluation also empirically confirms our theoretical claim in Section~\ref{sec:method}: while demonstrations are used during training to enable implicit quality reweighting, In-Context RLVR implicitly optimizes the base policy which can operate without demonstrations.

\subsection{Details for Training Dynamics Analysis}
\label{app:analysis_details}

Theoretical analysis in Section~\ref{sec:method} shows that In-Context RLVR implicitly reweights rewards by Evidence Gain, with the reweighted reward $R \cdot w$ implying a two-stage selection mechanism: first, the binary reward $R$ filters out traces with incorrect answers ($R=0$); second, among correct traces ($R=1$), the weight $w \propto \exp(\Delta)$ differentiates reasoning quality. Therefore, Evidence Gain is designed to distinguish reasoning quality \emph{among correct solutions}. To validate this implicit reweighting mechanism empirically (Section~\ref{sec:analysis}), we track both Evidence Gain and reasoning quality scores exclusively on traces with correct final answers.

Specifically, we use checkpoints from both DAPO and IC-DAPO at 1.5B and 7B scales. Every 100 training steps, we randomly sample 100 queries from the training set. Importantly, at each chosen step, DAPO and IC-DAPO share the same set of sampled queries to ensure fair comparison. For each query, we generate 8 rollouts and retain only those with correct final answers. We then compute Evidence Gain on the held-out set $\mathcal{E}_0$, and assess reasoning quality using DeepSeek-V3.2 as described in Appendix~\ref{app:llm_eval}. Finally, we compute the Spearman correlation $\rho$ between Evidence Gain and quality scores to verify that Evidence Gain remains a valid quality signal throughout training.

\subsection{LLM-based Quality Evaluation}
\label{app:llm_eval}

To automatically assess reasoning quality, we employ DeepSeek-V3.2~\citep{deepseekai2025deepseekv32pushingfrontieropen} as an LLM-based evaluator. Our evaluation rubric is informed by prior work on solution quality assessment~\citep{ye2024flaskfinegrainedlanguagemodel, xia2025reasoningeval, olgaroscoe}. To ensure comprehensive coverage, we define eight complementary dimensions as shown in Table~\ref{tab:eval_dimensions}. For each reasoning trace, the evaluator assigns a score from 1 to 5 on each dimension along with explicit textual explanations for justification, and finally provides an overall quality score from 1 to 5. The complete prompt template is provided in Appendix~\ref{sec:quality_eval_prompt}. 

\begin{table}[h]
\centering
\small
\begin{tabularx}{\columnwidth}{lX}
\toprule
\textbf{Dimension} & \textbf{Definition} \\
\midrule
Repetition & Same steps or ideas repeated \\
Redundancy & Unnecessary or verbose content \\
Logical Consistency & Contradictions \\
Relevance & Off-topic exploration \\
CoT-Ans Alignment & Answer derived from reasoning \\
Reasoning Rigor & Claims justified without leaps \\
Clarity & Easy to follow, well-structured \\
Completeness & All necessary steps present \\
\bottomrule
\end{tabularx}
\caption{Quality evaluation dimensions.}
\label{tab:eval_dimensions}
\end{table}

\subsection{Implementation of IC-DAPO}
Following standard DAPO, we sample 8 rollouts per training question. To integrate In-Context RLVR, for $m$ of these rollouts, we independently sample a random demonstration from $\mathcal{E}$ and prepend it to the question before generation; the remaining $n = 8 - m$ rollouts are generated from the original question. All 8 rollouts are then scored and updated using DAPO objective. In our experiments, we set $m = 6$ and $n = 2$. This mixed sampling strategy serves two purposes: (1) it preserves sufficient reward variance within each group for DAPO's group-based normalization to compute meaningful advantages, and (2) it increases the diversity of input contexts within each training batch. Our experiments validate the effectiveness of this configuration. Practitioners may explore alternative values of $m$ to adapt to different tasks.

\subsection{Implementation of Baselines}
\paragraph{GRPO} optimizes policies via group-based reward normalization without requiring a separate critic model. Following \citet{shao2024deepseekmathpushinglimitsmathematical}, we adopt symmetric clipping bounds with $\epsilon = 0.2$.
\paragraph{DAPO} extends GRPO by introducing asymmetric clipping bounds and dynamic sample filtering to mitigate entropy collapse. Following \citet{yu2025dapoopensourcellmreinforcement}, we set the lower and upper clipping thresholds to $\epsilon_{\text{low}} = 0.2$ and $\epsilon_{\text{high}} = 0.28$, respectively.
\paragraph{CISPO} applies clipping directly to the importance sampling weights rather than to the final policy update. Following \citet{cui2025entropymechanismreinforcementlearning}, we set symmetric clipping bounds with $\epsilon = 0.2$.
\paragraph{GSPO} employs a sequence-level importance ratio to enhance training stability and scalability. Following \citet{zheng2025groupsequencepolicyoptimization}, we set the lower and upper clipping thresholds to $\epsilon_{\text{low}} = 0.0003$ and $\epsilon_{\text{high}} = 0.0004$, respectively.
\paragraph{CE-GPPO} reintroduces gradient signals from tokens outside the clipping interval in a bounded manner through a stop-gradient operation, enabling fine-grained control over policy entropy dynamics. We directly report results from the original paper~\citep{su2025cegppocoordinatingentropygradientpreserving}. As this work provides two sets of evaluation results under different configurations ($\beta_1 = 0.5, \beta_2 = 1$ and $\beta_1 = 0.75, \beta_2 = 1$), we report their average scores in Table~\ref{tab:table1}.

\section{Additional Experiments}
\label{app:add_exp}

\subsection{What Does Evidence Gain Capture?}
\label{app:evidence_gain_validation}

\begin{figure}[t]
\centering
\includegraphics[width=\columnwidth]{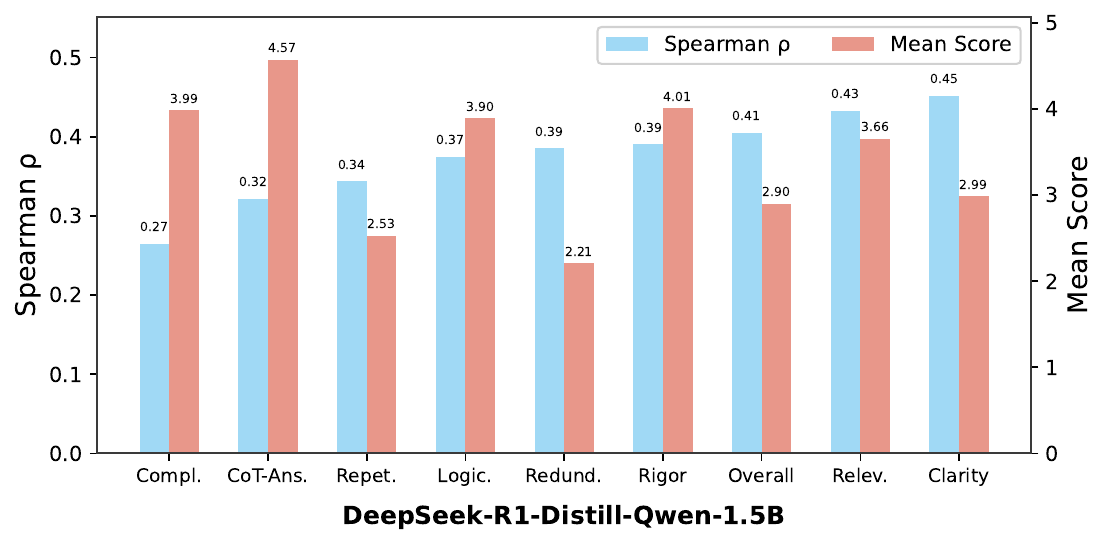}
\caption{\lightblue{Spearman} correlation between Evidence Gain and individual quality dimensions, alongside \lightred{mean scores} for each dimension. Evaluated on DeepSeek-R1-Distill-Qwen-1.5B with the same setup as Section~\ref{sec:evidence_gain}.}
\label{fig:dim_corr}
\end{figure}
To investigate which aspects of reasoning quality Evidence Gain most effectively captures, we conduct a fine-grained correlation analysis across the eight dimensions defined in our evaluation rubric (Appendix~\ref{app:llm_eval}). Figure~\ref{fig:dim_corr} presents results on DeepSeek-R1-Distill-Qwen-1.5B. We observe that all correlations are positive, ranging from $\rho = 0.27$ to $\rho = 0.45$, suggesting that Evidence Gain reflects multiple aspects of reasoning quality rather than a single dimension. 

Notably, Clarity ($\rho = 0.45$) and Relevance ($\rho = 0.43$) show stronger correlations than Completeness ($\rho = 0.27$) and CoT-Answer Alignment ($\rho = 0.32$). This pattern can be explained by our setup. Since we only evaluate traces that arrive at correct answers (Appendix \ref{app:exp_details}), these traces are already filtered for answer validity. As shown by the mean scores in Figure~\ref{fig:dim_corr}, Completeness (mean 3.99) and CoT-Answer Alignment (mean 4.57) exhibit high scores with limited variance among correct traces, because reaching the right answer typically requires including necessary steps and properly deriving the conclusion. In contrast, Clarity (mean 2.99) and Relevance (mean 3.66) show lower mean scores even among correct solutions, since a trace can reach the right answer while still being poorly organized. This difference in variance explains why Evidence Gain shows stronger correlations with dimensions that have greater room to discriminate among correct traces.

\subsection{Comparison with Proxy Signals}
\label{app:proxy_signals}

This section provides detailed descriptions of the three proxy signals compared with Evidence Gain in Section~\ref{sec:evidence_gain}.

\begin{itemize}
    \item \textbf{Length}: token count of the reasoning trace.
    \item \textbf{LogProb}: log-probability of rollouts, reflecting the model's confidence in its own output.
    \item \textbf{MajorVote}: a binary indicator of whether the answer matches the majority answer across rollouts, measuring answer-level consistency.
\end{itemize}

\begin{table*}[t]
\centering
\setlength{\tabcolsep}{14pt}
\resizebox{1.0\linewidth}{!}{
\begin{tabular}{lccccccc}
\toprule
\textbf{Method} & \textbf{AIME24} & \textbf{AIME25} & \textbf{HMMT25} & \textbf{MATH500} & \textbf{AMC23} & \textbf{Olympiad} & \textbf{Average} \\
\midrule
\textbf{DS-R1-Distill-Qwen-1.5B} & 29.2 & 24.1 & 13.1 & 86.0 & 73.7 & 51.8 & 46.3 \\
\quad + DAPO~\citep{yu2025dapoopensourcellmreinforcement}  & 40.0 & 28.4 & 19.2 & 90.0 & 84.4 & 61.6 & 53.9\\
\quad + IC-DAPO (V3.1) & \underline{44.5} \rlap{\red{$\boldsymbol\uparrow$}}& \underline{32.3} \rlap{\red{$\boldsymbol\uparrow$}}& \underline{19.5} \rlap{\red{$\boldsymbol\uparrow$}}& \underline{90.3} \rlap{\red{$\boldsymbol\uparrow$}}& \underline{85.8} \rlap{\red{$\boldsymbol\uparrow$}}& \underline{61.7} \rlap{\red{$\boldsymbol\uparrow$}}& \underline{55.7} \rlap{\red{$\boldsymbol\uparrow$}}\\
\quad + IC-DAPO (R1) & \textbf{45.6} \rlap{\red{$\boldsymbol\uparrow$}}& \textbf{34.2} \rlap{\red{$\boldsymbol\uparrow$}}& \textbf{19.7} \rlap{\red{$\boldsymbol\uparrow$}}& \textbf{90.6} \rlap{\red{$\boldsymbol\uparrow$}}& \textbf{86.2} \rlap{\red{$\boldsymbol\uparrow$}}& \textbf{62.1} \rlap{\red{$\boldsymbol\uparrow$}}& \textbf{56.4} \rlap{\red{$\boldsymbol\uparrow$}} \\
\midrule
\textbf{DS-R1-Distill-Qwen-7B} & 54.5 & 39.1 & 26.2 & 93.6 & 90.6 & 67.0 & 61.8  \\
\quad + DAPO~\citep{yu2025dapoopensourcellmreinforcement}  & 62.0 & 45.9 & 27.4 & 94.1 & 92.3 & 69.9 & 65.3 \\
\quad + IC-DAPO (V3.1)  & \underline{63.3} \rlap{\red{$\boldsymbol\uparrow$}} & \underline{47.5} \rlap{\red{$\boldsymbol\uparrow$}} & \underline{29.2} \rlap{\red{$\boldsymbol\uparrow$}}&\underline{95.5} \rlap{\red{$\boldsymbol\uparrow$}} &  \underline{92.6} \rlap{\red{$\boldsymbol\uparrow$}} & \underline{70.8} \rlap{\red{$\boldsymbol\uparrow$}} & \underline{66.4} \rlap{\red{$\boldsymbol\uparrow$}}\\
\quad + IC-DAPO (R1)  & \textbf{66.5} \rlap{\red{$\boldsymbol\uparrow$}} & \textbf{49.8} \rlap{\red{$\boldsymbol\uparrow$}} & \textbf{29.4} \rlap{\red{$\boldsymbol\uparrow$}}&\textbf{95.6} \rlap{\red{$\boldsymbol\uparrow$}} &  \textbf{93.7} \rlap{\red{$\boldsymbol\uparrow$}} & \textbf{71.7} \rlap{\red{$\boldsymbol\uparrow$}} & \textbf{67.8} \rlap{\red{$\boldsymbol\uparrow$}}\\
\bottomrule
\end{tabular}
}
\caption{Ablation on demonstration quality. IC-DAPO (R1) uses refined reasoning traces from DeepSeek-R1, while IC-DAPO (V3.1) uses solutions from DeepSeek-V3.1, a non-reasoning model. Both variants outperform DAPO, with R1 demonstrations yielding stronger results. \textbf{Bold} and \underline{underline} indicate best and second-best results. \red{$\boldsymbol\uparrow$} denotes improvement over DAPO.}
\label{tab:demo_quality}
\end{table*}

\subsection{Human Evaluation}
\label{app:human_eval}
We further validate whether the automatic quality scores produced by DeepSeek-V3.2 are consistent with human judgement. We first randomly sample 100 question-reasoning pairs generated by Deepseek-R1-Distill-Qwen-1.5B whose final answers are verified correct by rule. We then ask \textbf{four} human experts to separately assign an overall quality score between 1 and 5 for each reasoning trace. Annotators follow the same rubric for Deepseek-V3.2 (Appendix~\ref{app:llm_eval}). Finally, we compute Spearman correlation coefficients between the DeepSeek-V3.2 scores and each human expert, as well as the correlations among the human experts.
\begin{figure}[t]
\centering
\includegraphics[width=\columnwidth]{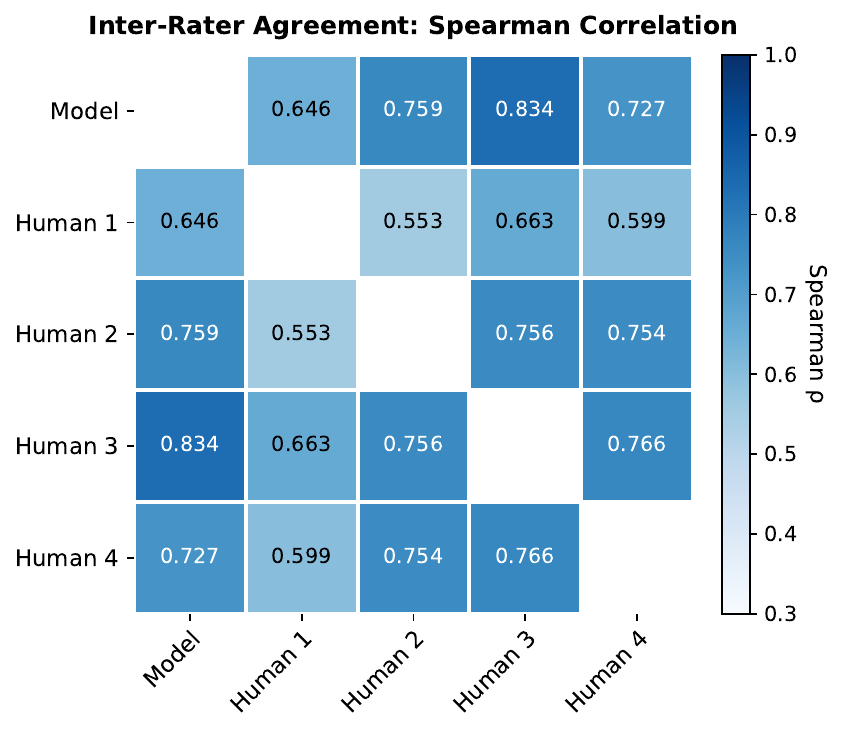}
\caption{Spearman correlation matrix between DeepSeek-V3.2 quality scores and four human expert ratings on 100 sampled reasoning traces, together with inter rater correlations among experts.}
\label{fig:human_corr}
\end{figure}

Figure~\ref{fig:human_corr} shows strong agreement between DeepSeek-V3.2 and human experts. The correlations between DeepSeek-V3.2 and individual experts fall between 0.65 and 0.83, with an average of 0.74, which is comparable to human expert agreement. Overall, DeepSeek-V3.2 performs within the variability of human judgement on this task, supporting its reliability as an automatic evaluator in our experiments. We attribute this robustness to the comprehensive and fine-grained quality rubric, which decomposes reasoning quality into multiple distinguishable dimensions and thus enables more consistent judgements across different raters.

\begin{table}[t]
\centering
\setlength{\tabcolsep}{7pt}
\resizebox{1.0\linewidth}{!}{
\begin{tabular}{llccccc}
\toprule
\textbf{Method} & \textbf{Scale} & \textbf{HumanEval} & \textbf{LCB} & \textbf{IFBench} & \textbf{IFEval} & \textbf{MMLU} \\
\midrule
DAPO & 1.5B & 73.6 & 30.4 & 9.1 & 38.3 & 46.6 \\
IC-DAPO & 1.5B & \textbf{75.5} & \textbf{31.9} & \textbf{10.7} & \textbf{40.6} & \textbf{48.3} \\
\midrule
DAPO & 7B & 90.5 & 48.5 & 13.3 & 56.7 & 65.7 \\
IC-DAPO & 7B & \textbf{92.3} & \textbf{49.8} & \textbf{14.2} & \textbf{58.6} & \textbf{67.3} \\
\bottomrule
\end{tabular}
}
\caption{Cross-domain zero-shot evaluation on code generation (HumanEval, LCB), instruction following (IFBench, IFEval), and general knowledge (MMLU).}
\label{tab:cross_domain}
\end{table}

\subsection{Cross-Domain Generalization}
While our training focuses on mathematical reasoning, the theoretical mechanism is domain-agnostic. To test cross-domain generalization, we evaluate IC-DAPO and DAPO checkpoints on 5 additional benchmarks spanning code generation (HumanEval, LiveCodeBench), instruction following (IFBench, IFEval), and general knowledge (MMLU) in zero-shot mode. Table~\ref{tab:cross_domain} shows that IC-DAPO consistently outperforms DAPO across all domains and both scales, suggesting that the benefits of quality-aware reweighting are not confined to mathematical reasoning.

\begin{figure}[t]
\centering
\includegraphics[width=\columnwidth]{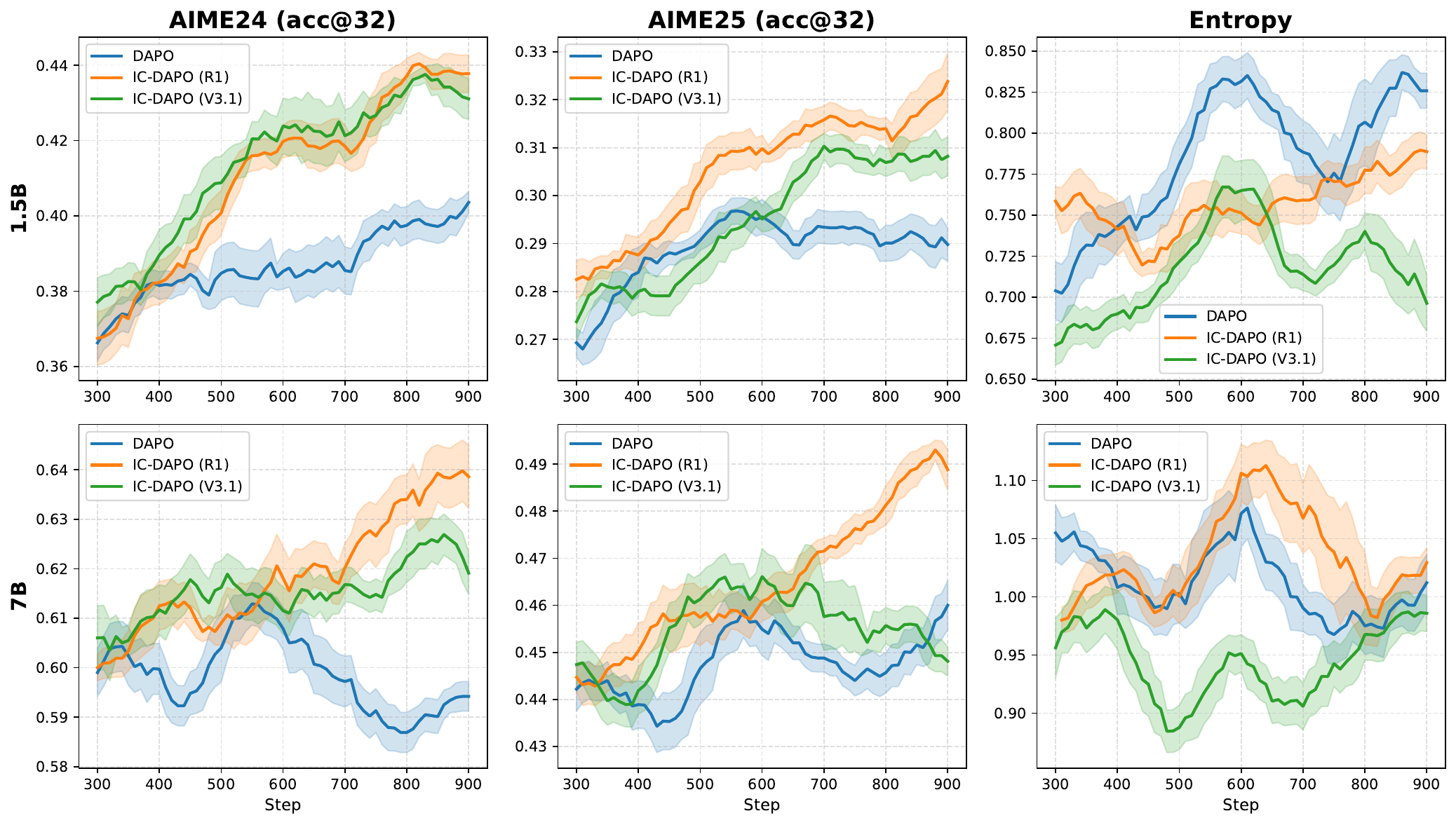}
\caption{Training dynamics across 1.5B and 7B models. IC-DAPO variants consistently outperform DAPO on AIME24 and AIME25 while maintaining stable entropy throughout training.}
\label{fig:training_dynamics}
\end{figure}

\subsection{Ablation Study on Demonstration Quality}
\label{app:ablation}

Our main experiments in Section~\ref{sec:exp} construct the demonstration set $\mathcal{E}$ using the refined content following DeepSeek-R1's \texttt{</think>} tag. To investigate whether demonstration quality affects training outcomes, we construct an alternative set using DeepSeek-V3.1~\citep{deepseekai2024deepseekv3technicalreport}, a strong non-reasoning model, for comparison. Without the reasoning capabilities that produce R1's refined traces, V3.1 is expected to generate solutions of lower quality, making it a suitable baseline for studying the effect of demonstration quality. To ensure fair comparison, we use the same 1,082 questions from the demonstration set $\mathcal{E}$ described in Section~\ref{sec:exp} and generate solutions using DeepSeek-V3.1. We generate multiple responses per question to ensure each question obtains a correct solution. 

Table~\ref{tab:demo_quality} shows that both IC-DAPO variants outperform DAPO, confirming In-Context RLVR's robustness over different demonstration sources. Crucially, IC-DAPO (R1) always surpasses IC-DAPO (V3.1), with gains of +1.1 on AIME24 and +1.9 on AIME25 at 1.5B scale. From a theoretical perspective, since the reference solutions in validation set are used to compute Evidence Gain (Eq.~\ref{eq:evidence_gain}), the quality of these references may affect the accuracy of $\Delta$ as a quality signal, which in turn affects the effectiveness of implicit reweighting (Eq. \ref{eq:reweighted_obj}).

\subsection{Training Dynamics}
\label{app:training_dynamics}

We present extended training dynamics for both 1.5B and 7B models in Figure~\ref{fig:training_dynamics}. Across both model scales, IC-DAPO variants (i.e., IC-DAPO (R1) and IC-DAPO (V3.1)) consistently achieve higher accuracy on AIME24 and AIME25 throughout training while maintaining comparable entropy trajectories to the DAPO baseline. Notably, IC-DAPO (R1) demonstrates a clear advantage over standard DAPO, with the performance gap widening as training progresses. Comparing IC-DAPO (R1) and IC-DAPO (V3.1), we observe that R1-generated demonstrations yield better final performance, consistent with our findings in Section~\ref{app:ablation} that higher-quality demonstrations lead to improved training outcomes. The entropy curves remain stable across all methods, indicating that In-Context RLVR does not compromise training stability while achieving superior accuracy, which aligns with the stability analysis in prior work~\citep{yu2025dapoopensourcellmreinforcement}.

\section{Quality Evaluation Prompt}
\label{sec:quality_eval_prompt}

We present the prompt template used for reasoning quality evaluation with DeepSeek-V3.2 in Figure \ref{fig:quality_eval_prompt}. The prompt assesses reasoning traces across eight dimensions, including Repetition, Redundancy, Logical Consistency, Relevance, CoT-Answer Alignment, Reasoning Rigor, Clarity, and Completeness, with each scoring on a 1--5 scale. The template uses placeholders \textcolor{placeholder}{\texttt{\{question\}}} and \textcolor{placeholder}{\texttt{\{response\}}} which are populated with the specific math problem and corresponding reasoning trace during evaluation.

\begin{figure*}[ht]
\begin{tcolorbox}[
    colback=codebg,
    colframe=codeframe,
    boxrule=0.6pt,
    arc=0pt,
    left=4pt,
    right=4pt,
    top=4pt,
    bottom=4pt
]
\begin{lstlisting}[style=promptstyle]
# Mathematical Reasoning Quality Evaluation

## Task
Evaluate the **quality** of mathematical reasoning (answer already verified correct). 
Assess whether this reasoning would be valuable as a learning reference.

**Key insight**: Correct answer (*@$\neq$@*) good reasoning. Watch for warning signs:
- Frequent "wait", "hold on", "let me try again"
- Same calculation repeated multiple times
- Long explorations that don't contribute to the answer
- Answer appearing without clear derivation

Be strict. Reasoning that wanders, second-guesses itself, or reaches the right answer through messy exploration is not high-quality, regardless of correctness.

## Dimensions (1-5 each)

| Dimension           | Assess                                      |
|---------------------|---------------------------------------------|
| Repetition          | Same steps or ideas repeated?               |
| Redundancy          | Unnecessary or verbose content?             |
| Logical Consistency | Contradictions or backtracks?               |
| Relevance           | Off-topic content or dead-end exploration?  |
| CoT-Answer Alignment| Answer clearly derived from reasoning?      |
| Reasoning Rigor     | All claims justified without leaps?         |
| Clarity             | Easy to follow and well-structured?         |
| Completeness        | All necessary steps present?                |

## Scoring

| Score | Meaning                                              |
|-------|------------------------------------------------------|
| 5     | Excellent: Perfect textbook quality                  |
| 4     | Good: Minor flaws, suitable as reference             |
| 3     | Average: Clear flaws, but followable                 |
| 2     | Weak: Major issues hurting pedagogical value         |
| 1     | Poor: Guessed answer, chaotic flow                   |

## Evaluate This

### Math Problem:
(*@\textcolor{placeholder}{\texttt{\{question\}}}@*)

### Reasoning Process (Answer Verified Correct):
(*@\textcolor{placeholder}{\texttt{\{response\}}}@*)

---

Output JSON only:
{
  "dimensions": {
    "repetition": {"score": <1-5>, "comment": "<evidence>"},
    "redundancy": {"score": <1-5>, "comment": "<evidence>"},
    "logical_consistency": {"score": <1-5>, "comment": "<evidence>"},
    "relevance": {"score": <1-5>, "comment": "<evidence>"},
    "cot_answer_alignment": {"score": <1-5>, "comment": "<evidence>"},
    "reasoning_rigor": {"score": <1-5>, "comment": "<evidence>"},
    "clarity": {"score": <1-5>, "comment": "<evidence>"},
    "completeness": {"score": <1-5>, "comment": "<evidence>"}
  },
  "overall_analysis": "<key findings, limiting factors>",
  "score": <1-5>
}
\end{lstlisting}
\end{tcolorbox}
\caption{Quality Evaluation Prompt Template. Placeholders \textcolor{placeholder}{\texttt{\{question\}}} and \textcolor{placeholder}{\texttt{\{response\}}} are replaced with the actual math problem and reasoning trace during evaluation.}
\label{fig:quality_eval_prompt}
\end{figure*}

\clearpage

\onecolumn

\section{Proofs of Main Results}
\label{app:proof}

{\small\color{gray}\noindent [This section of appendix uses single-column format for mathematical readability.]}\\

This appendix provides complete derivations for the theoretical results in Section~\ref{sec:method}. We first establish notation, then derive the reweighting identity and characterize its relationship to Evidence Gain.

\subsection{Notations}
\label{app:notation}

We formalize the training setup. Let $q \sim \mathcal{D}$ denote a question from the training distribution. Let $\mathcal{E} = \{e^{(i)}\}_{i=1}^{|\mathcal{E}|}$ be the held-out validation set, where each demonstration $e = (e_q, e_r)$ consists of a question $e_q$ and a high-quality reference reasoning trace $e_r$. During In-Context RLVR training, a demonstration $e$ is sampled from $\mathcal{E}$ and prepended to $q$, after which the model generates a reasoning trace $r \sim \pi_\theta(\cdot | e, q)$.

\subsection{Bayesian Identity}
\label{app:bayesian_identity}

We establish the key identity relating the conditioned policy to the base policy. The derivation relies on the following assumption, which reflects the independent sampling structure in our data construction.

\begin{assumption}
\label{assump:independence}
Providing only the demonstration question $e_q$, without its reasoning trace $e_r$, does not alter the distribution over reasoning traces for a training question $q$. Conversely, providing only the training question $q$, without any reasoning trace, does not alter the distribution over reasoning traces for the demonstration question $e_q$. Formally:
\begin{align}
    \pi_\theta(r | e_q, q) &= \pi_\theta(r | q), \tag{A1}\\
    \pi_\theta(e_r | q, e_q) &= \pi_\theta(e_r | e_q). \tag{A2}
\end{align}
\end{assumption}

\paragraph{Remark.} 
This assumption is natural given the independent sampling of demonstrations and training questions. To see why (A1) holds, suppose $e_q$ is ``Solve $x^2 - 5x + 6 = 0$'' and $q$ is ``Compute $\int \sin x \, dx$.'' The bare statement of $e_q$ carries no information about integration techniques; it only indicates that the context involves math. Crucially, the model already knows this from observing $q$ itself. Thus, conditioning on $e_q$ alone provides no additional signal for solving $q$. A symmetric argument establishes (A2). While edge cases may exist where two questions happen to share methodological structure, so that $e_q$ could bias the preferred style or method for solving $q$, such coincidences are rare under independent sampling and average out at scale. Thus (A1) and (A2) hold as statistical approximations.

\paragraph{Empirical Validation.} 
We validate Assumption~\ref{assump:independence} using DeepSeek-R1-Distill-Qwen-1.5B. We randomly select 100 question-reasoning pairs $(q, r)$ from rollouts generated during training, and independently sample 100 additional questions $\{q'\}$ from the dataset. We prepend each $q'$ to each $(q, r)$ resulting in 10,000 samples ($S$). We measure the relative change in reasoning log-probability when prepending an randomly select question $q'$:
\begin{equation}
\delta = \frac{1}{|S|}\sum_{(q,r,q')\in S} \frac{|\log \pi_\theta(r|q',q) - \log \pi_\theta(r|q)|}{|\log \pi_\theta(r|q)|}.
\end{equation}
Notably, we obtain $\boldsymbol{\delta = 0.0384} < 5\%$, confirming that independently sampled question statements have negligible influence on reasoning distributions on average.

Although Assumption~\ref{assump:independence} suggests that question statements alone provide negligible cross example influence, a complete demonstration $(e_q, e_r)$ \emph{does} provide transferable information: the reasoning trace $e_r$ may exhibit problem-solving patterns (e.g., algebraic manipulation, problem decomposition) that generalize across problems. This distinction is precisely what Evidence Gain captures, and it explains why the policy model's ICL ability can serve as an effective quality signal.

\begin{lemma}[Bayesian Identity]
\label{lemma:bayesian}
Under Assumption~\ref{assump:independence}, the conditioned policy admits the decomposition:
\begin{equation}
    \pi_\theta(r|e,q) = \pi_\theta(r|q) \cdot \frac{\pi_\theta(e_r|q, r, e_q)}{\pi_\theta(e_r|e_q)}.
    \label{eq:bayesian_identity_app}
\end{equation}
\end{lemma}

\begin{proof}
We begin by expanding the conditioned policy $\pi_\theta(r | e_q, e_r, q)$ using the definition of conditional probability. By Bayes' Rule, we have:
\begin{equation}
    \pi_\theta(r | e_q, e_r, q)
    = \frac{\pi_\theta(r | e_q, q)\, \pi_\theta(e_r | q, r, e_q)}{\pi_\theta(e_r | q, e_q)}.
    \label{eq:bayes_expanded}
\end{equation}
We now apply Assumption~\ref{assump:independence} to simplify the right-hand side. By (A1), the numerator term $\pi_\theta(r | e_q, q)$ reduces to $\pi_\theta(r | q)$, since observing $e_q$ alone provides no additional information for generating $r$. By (A2), the denominator term $\pi_\theta(e_r | q, e_q)$ reduces to $\pi_\theta(e_r | e_q)$, since observing $q$ alone provides no additional information for generating $e_r$. Substituting these simplifications into Eq.~\eqref{eq:bayes_expanded}, we obtain:
\begin{equation}
    \pi_\theta(r | e_q, e_r, q)
    = \pi_\theta(r | q) \cdot \frac{\pi_\theta(e_r | q, r, e_q)}{\pi_\theta(e_r | e_q)}.
\end{equation}
This completes the proof.
\end{proof}

\subsection{Implicit Reward Reweighting}
\label{app:reweight_proof}

We now establish that In-Context RLVR implicitly performs reward reweighting, showing how the policy model's ICL mechanism naturally upweights high-quality reasoning traces. We present the theoretical results in two separate theorems.

\begin{theorem}[Implicit Reweighting]
\label{thm:implicit_reweight}
Under Assumption~\ref{assump:independence}, the In-Context RLVR objective
\begin{equation}
    J_{\textup{IC}}(\theta) = \mathbb{E}_{q} \mathbb{E}_{e \sim \mathcal{E},\, r \sim \pi_\theta(\cdot|e,q)} \bigl[ R(q, r) \bigr]
\end{equation}
can be exactly rewritten as
\begin{equation}
    J_{\textup{IC}}(\theta) = \mathbb{E}_{q} \mathbb{E}_{r \sim \pi_\theta(\cdot|q)} \bigl[ R(q, r) \cdot w(q, r) \bigr],
\end{equation}
where the weight factor is defined as
\begin{equation}
    w(q,r) = \mathbb{E}_{e \sim \mathcal{E}}[\exp(\Delta_e)]
\end{equation}
with $\Delta_e = \log \pi_\theta(e_r|q,r,e_q) - \log \pi_\theta(e_r|e_q)$.
\end{theorem}

\begin{proof}
Invoking Lemma~\ref{lemma:bayesian} and assuming uniform sampling over $\mathcal{E}$, we expand:
\begin{align}
    J_{\textup{IC}}(\theta) 
    &= \mathbb{E}_{q} \left[ \frac{1}{|\mathcal{E}|} \sum_{e} \sum_{r} \pi_\theta(r|e,q) \cdot R(q, r) \right] \\[4pt]
    &= \mathbb{E}_{q} \left[ \frac{1}{|\mathcal{E}|} \sum_{e} \sum_{r} \pi_\theta(r|q) \cdot \frac{\pi_\theta(e_r|q,r,e_q)}{\pi_\theta(e_r|e_q)} \cdot R(q,r) \right] \\[4pt]
    &= \mathbb{E}_{q} \left[ \sum_{r} \pi_\theta(r|q) \cdot R(q,r) \cdot \underbrace{\frac{1}{|\mathcal{E}|}\sum_{e} \frac{\pi_\theta(e_r|q,r,e_q)}{\pi_\theta(e_r|e_q)}}_{w(q,r)} \right].
\end{align}
Writing $\Delta_e = \log \pi_\theta(e_r|q,r,e_q) - \log \pi_\theta(e_r|e_q)$, the weight becomes $w(q,r) = \mathbb{E}_{e}[\exp(\Delta_e)]$.
\end{proof}

\paragraph{Interpretation.} Theorem~\ref{thm:implicit_reweight} establishes an \emph{exact} equivalence: the In-Context RLVR objective is mathematically identical to standard RLVR with rewards reweighted by $w(q,r)$. No approximation is involved. The weight $w(q,r) = \mathbb{E}_e[\exp(\Delta_e)]$ measures how much the candidate trace $r$ improves the model's ability to generate reference solutions on average.

The next theorem characterizes the relationship between the implicit weight $w(q,r)$ and the Evidence Gain $\Delta(q,r) = \mathbb{E}_e[\Delta_e]$.

\begin{theorem}[Log-Linear Approximation]
\label{thm:log_linear}
The weight factor $w(q,r)$ from Theorem~\ref{thm:implicit_reweight} and the Evidence Gain $\Delta(q,r) = \mathbb{E}_e[\Delta_e]$ satisfy:
\begin{enumerate}
    \item \textbf{(Lower bound)} By Jensen's inequality: $w(q,r) \geq \exp\bigl(\Delta(q,r)\bigr)$.
    \item \textbf{(Refined bound)} $\log w(q,r) = \Delta(q,r) + \log\bigl(1 + \tfrac{1}{2}\mathrm{Var}_e[\Delta_e]\bigr) + o(\mathrm{Var}_e[\Delta_e])$.
\end{enumerate}
In particular, when $\mathrm{Var}_e[\Delta_e]$ is approximately constant across $(q,r)$ pairs, the relationship simplifies to:
\begin{equation}
    \log w(q, r) \approx \Delta(q, r) + c,
\end{equation}
for some constant $c$ that depends on the average variance $\mathbb{E}_{q,r}[\mathrm{Var}_e[\Delta_e]]$.
\end{theorem}

\begin{proof}
Since $\exp(\cdot)$ is convex, Jensen's inequality gives $\mathbb{E}_e[\exp(\Delta_e)] \geq \exp(\mathbb{E}_e[\Delta_e]) = \exp(\Delta(q,r))$, establishing (i).

For (ii), we expand $\exp(\Delta_e)$ around $\Delta:=\Delta(q,r)$ via Taylor series:
\begin{align}
    \exp(\Delta_e) &= \exp(\Delta) \cdot \exp(\Delta_e - \Delta) \\
    &= \exp(\Delta) \cdot \Bigl(1 + (\Delta_e - \Delta) + \tfrac{1}{2}(\Delta_e - \Delta)^2 + o\bigl((\Delta_e - \Delta)^2\bigr)\Bigr).
\end{align}
Taking expectations and using $\mathbb{E}_e[\Delta_e - \Delta] = 0$:
\begin{equation}
    w(q,r) = \exp(\Delta) \cdot \bigl(1 + \tfrac{1}{2}\mathrm{Var}_e[\Delta_e] + o(\mathrm{Var}_e[\Delta_e])\bigr).
\end{equation}
Taking logarithms on both sides and applying $\log(1+x) = x + o(x)$ yields (ii).
\end{proof}

\paragraph{Interpretation.} Result (i) shows that $\exp(\Delta(q,r))$ serves as a \emph{lower bound} for the implicit weight $w(q,r)$, but does not quantify how tight this bound is. Result (ii) refines this by showing that the gap is controlled by $\tfrac{1}{2}\mathrm{Var}_e[\Delta_e]$. Specifically, $\log w(q,r)$ exceeds $\Delta(q,r)$ by approximately $\log(1 + \tfrac{1}{2}\mathrm{Var}_e[\Delta_e]) > 0$, a strictly positive correction.

\begin{figure}[t]
\centering
\includegraphics[width=\columnwidth]{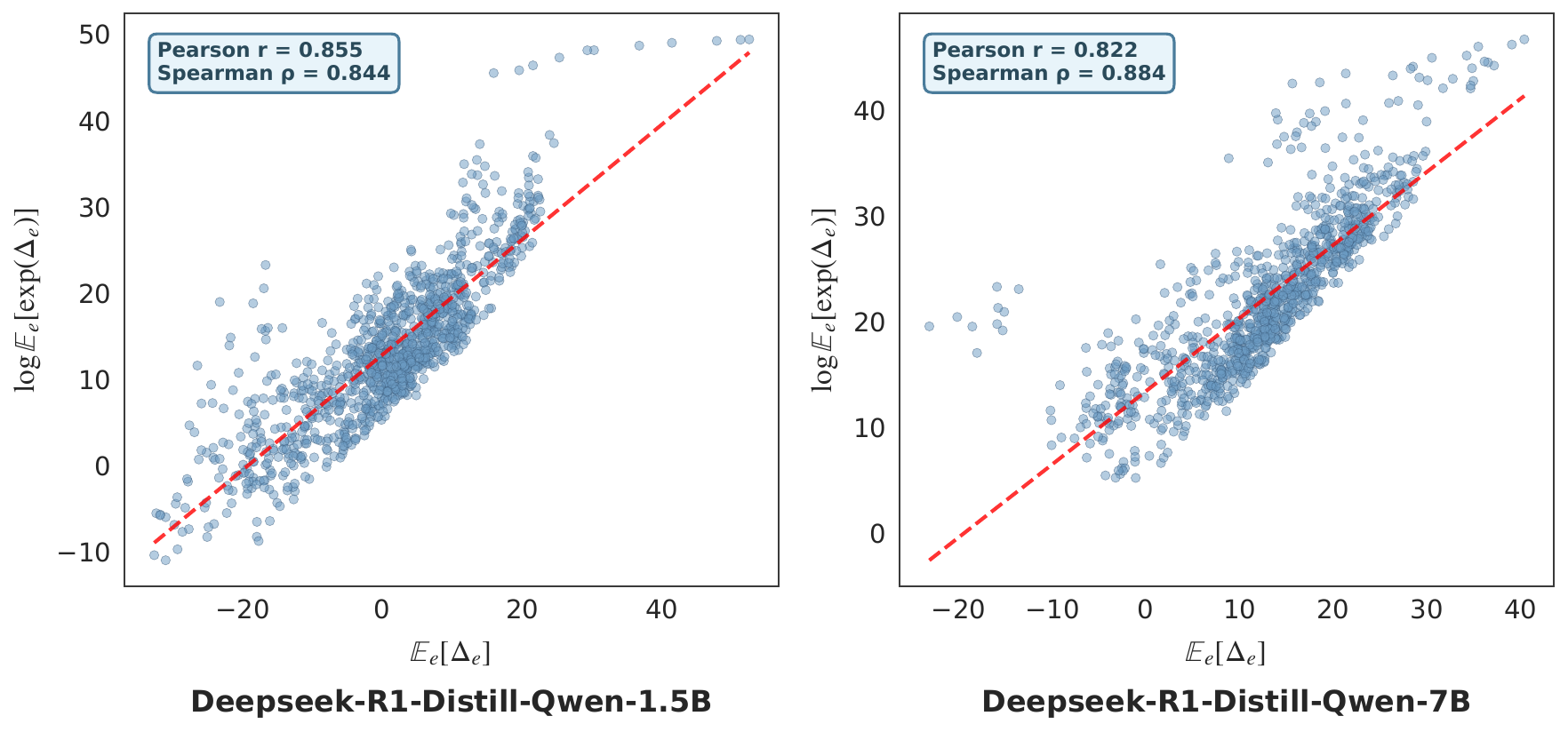}
\caption{Empirical verification of the log-linear relationship between $\log w(q,r) = \log \mathbb{E}_e[\exp(\Delta_e)]$ and $\Delta(q,r) = \mathbb{E}_e[\Delta_e]$. Strong Pearson correlations ($r = 0.855$ for 1.5B, $r = 0.822$ for 7B) confirm that Evidence Gain serves as a reliable proxy for the implicit weight.}
\label{fig:log_linear}
\end{figure}

\paragraph{Empirical Verification.}
Although Theorem~\ref{thm:log_linear}(ii) indicates that $\log w(q,r)$ and $\Delta(q,r)$ differ by a variance-dependent term, if $\mathrm{Var}_e[\Delta_e]$ remains relatively stable across different $(q,r)$ pairs, the relationship simplifies to an approximate linear correspondence. To verify this, we conduct experiments using rollouts generated by DeepSeek-R1-Distill-Qwen at 1.5B and 7B~\citep{guo2025deepseekr1}. Specifically, we randomly sample 1,100 $(q, r)$ pairs for the 1.5B model and 1,000 pairs for the 7B model. For each $(q,r)$ pair, we compute $\Delta(q,r) = \mathbb{E}_e[\Delta_e]$ and $\log w(q,r) = \log \mathbb{E}_e[\exp(\Delta_e)]$, then measure their correlation.

Figure~\ref{fig:log_linear} presents the results. We observe strong linear relationships. The 1.5B model yields Pearson $r = 0.855$ (Spearman $\rho = 0.844$), while the 7B model achieves $r = 0.822$ (Spearman $\rho = 0.884$). These high correlations confirm that the variance term contributes a near-constant offset across $(q,r)$ pairs, validating the log-linear approximation in Theorem~\ref{thm:log_linear}.

\paragraph{Summary.} Combining Theorems~\ref{thm:implicit_reweight} and~\ref{thm:log_linear}, we conclude that In-Context RLVR exactly reweights rewards by $w(q,r) = \mathbb{E}_e[\exp(\Delta_e)]$, and this weight is approximately log-linear in Evidence Gain: $\log w \approx \Delta + c$. This confirms that the policy model's intrinsic ICL ability provides an effective quality signal. Traces with higher Evidence Gain receive proportionally higher weights in the reweighted objective, without requiring any external evaluator. In this way, In-Context RLVR leverages the model's own capacity to distinguish reasoning quality, enabling implicit reward reweighting through a simple modification to the training procedure.

\end{document}